\documentclass[letterpaper, 10 pt, conference]{ieeeconf}  

\IEEEoverridecommandlockouts                              

\usepackage{amsmath} 
\usepackage{amssymb}  

\usepackage{mathtools}
\usepackage{subcaption}
\usepackage[ruled,vlined]{algorithm2e}
\usepackage{url}
\usepackage{float}

\usepackage{graphicx}
\usepackage{caption}
\usepackage{booktabs}
\usepackage{varwidth}
\usepackage{multirow}
\usepackage{cite}

\usepackage{caption}
\usepackage{subcaption}

\makeatletter
\DeclareRobustCommand{\rvdots}{%
  \vbox{
    \baselineskip4\p@\lineskiplimit\z@
    \kern-\p@
    \hbox{.}\hbox{.}\hbox{.}
  }}
\makeatother

\DeclareMathOperator*{\argmin}{argmin\,}
\DeclareMathOperator*{\minimize}{minimize\,}

\DeclareMathOperator*{\st}{subject\,to\,}
\DeclareMathOperator*{\sign}{sign}

\DeclareMathOperator\trace{tr}
\DeclareMathOperator\diag{diag}

\newtheorem{theorem}{Theorem}
\newtheorem{corollary}{Corollary}
\newtheorem{lemma}{Lemma}

\newtheorem{definition}{Definition}

\newtheorem{assumption}{Assumption}

\title{\LARGE \bf
Distributed Estimation of Sparse Inverse Covariances}

\author{Tong Yao and Shreyas Sundaram 
\thanks{Tong Yao and Shreyas Sundaram are with the Elmore Family School of Electrical and Computer Engineering at Purdue University. 
Email: \{yao127,  sundara2
\}@purdue.edu
}  
\thanks{This research was supported by National Science Foundation (NSF) grant CMMI 1638311.
}
}

\begin{document}

\maketitle
\thispagestyle{empty}
\pagestyle{empty}

\begin{abstract}
Learning the relationships between various entities from time-series data is essential in many applications. Gaussian graphical models have been studied to infer these relationships. However, existing algorithms process data in a batch at a central location, limiting their applications in scenarios where data is gathered by different agents. In this paper, we propose a distributed sparse inverse covariance algorithm to learn the network structure (i.e., dependencies among observed entities) in real-time from data collected by distributed agents. Our approach is built on an online graphical alternating minimization algorithm, augmented with a consensus term that allows agents to learn the desired structure cooperatively. We allow the system designer to select the number of communication rounds and optimization steps per data point. We characterize the rate of convergence of our algorithm and provide simulations on synthetic datasets.
\end{abstract}

\section{INTRODUCTION}

Several applications involve analyzing many interacting entities, each generating a large quantity of multi-variate time-series data. Learning the relationships among these entities, especially in real-time, is essential for pattern discovery, prediction, and correlation-based clustering. In distributed scenarios, each agent (observer) can only collect data from a subset of the entities and communicate with its neighbors within a specific communication range to learn the relationships of the entire network in real-time (i.e., online). Such distributed online correlation inference has many applications, for example, in covariance-based and distributed clustering \cite{cov_cluster, dist_infer_sensor}, traffic network prediction and tracking \cite{traffic_lasso, traffic_tracking}, social network inference \cite{dist_infer_social}, and biomedical sensor data inference \cite{bio}. 

To learn the relationships between entities, one approach is to represent each entity as a node in a graph, and each edge defines an interaction between the nodes. Learning the network topology from data is studied extensively in the literature \cite{pavez2018learning,hassan2016topology, materassi2012problem, ayazoglu2011blind}. One method focuses on estimating the non-zero elements of the sparse inverse covariance matrix of these random variables (i.e., nodes) \cite{pavez2018learning, sojoudi2016equivalence} through inferring the edges of the graph by using an $l1$-regularized Gaussian maximum likelihood estimator, assuming the random variables are jointly Gaussian. In this case, an edge between nodes indicates that the corresponding random variables are conditionally dependent, given all the other variables. This is known as the graphical lasso problem \cite{yuan2007model,friedman2008sparse, banerjee2008model}. Various optimization algorithms have been proposed to solve this problem, including coordinate descent \cite{friedman2008sparse}, proximal methods \cite{rolfs2012iterative, hsieh2011sparse}, and alternating minimization methods  \cite{scheinberg2010sparse, dalal2017sparse}.
However, these algorithms have not been studied extensively for distributed and online inference. In \cite{OGAMA}, an online sparse inverse covariance algorithm was proposed, but it is not directly applicable to the distributed setting. 

In this paper, given multi-variate time-series data gathered by each agent in a network, we propose a peer-to-peer distributed algorithm for each agent to estimate the underlying relationships between all the network variables in real-time via distributed and online inverse covariance matrix estimation. 
We provide theoretical guarantees on the convergence of the estimates and characterize the asymptotic rate of convergence. We also demonstrate the performance of our proposed algorithm through simulations. 
\section{PROBLEM FORMULATION} \label{sec: problem formulation}

Consider a set of $p$ random variables $X = \{x_1, x_2, \ldots,x_p\}$ that are jointly Gaussian with zero mean and covariance $S^*$.  These variables are represented by a graph $\mathcal{G} = (\mathcal{V},\mathcal{E})$, where $\mathcal{V} = \{v_1,\ldots,v_p\}$ is the set of nodes, with each node $v_i$ representing a random variable $x_i$. An edge $(v_i,v_j) \in \mathcal{E}$ indicates that variable $x_j$ is conditionally dependent on $x_i$, given all the other random variables. Conversely, if $(v_i,v_j) \notin \mathcal{E}$, $v_j$ is conditionally independent of $v_i$, given all the other variables. This lack of an edge corresponds to a zero-entry in the inverse covariance matrix $(S^*)^{-1}$ (e.g., see \cite{friedman2008sparse,scheinberg2010sparse}).

These relationships (i.e., graph structure) between the variables are unknown {\it a priori}; the goal is to infer the edges of the graph based on samples of the random variables.  Specifically, at each time step $t \in \{1,2,\ldots\}$, the network generates data $X_t = \left[\begin{matrix}x_{1t} & x_{2t} & \cdots & x_{pt}\end{matrix}\right]^T \in \mathbb{R}^p$. We assume each $X_t$ is independently and identically sampled from the underlying Gaussian distribution, i.e., $X_t \sim \mathcal{N}(0,S^*)$.

Consider a group of $n$ agents, $n \in \mathbb{N}$, with a known, unweighted, undirected, and connected communication graph $\mathcal{G}_a = (\mathcal{V}_a,\mathcal{E}_a)$, where $\mathcal{V}_a = \{1,2,\ldots,n\}$ is the set of vertices representing the agents and $\mathcal{E}_a \subseteq \mathcal{V}_a \times \mathcal{V}_a$ is the set of edges. Note that the communication graph $\mathcal{G}_a$ is not to be confused with the correlation graph $\mathcal{G}$. If an edge $({i},{j}) \in \mathcal{E}_a$, agent $i$ and $j$ can communicate with each other. The neighbors of agent $i \in \mathcal{V}_a$ in graph $\mathcal{G}_a$ are represented by the set $N_i = \{j \in \mathcal{V}_a: (i,j) \in \mathcal{E}_a\}$. We let $\mathcal{N}_i$ denote the set of neighbors and agent $i$ itself, i.e., $\mathcal{N}_i = N_i \cup \{i\}$. Each agent $i \in \mathcal{V}_a$ observes a subset of the random variables $X_i \subseteq X$ and collects the corresponding time-series data $X_{i,t} \subseteq X_t$. The time-series data collected by the neighbors of agent $i$ is denoted by $X_{N_i,t} \subseteq X_{t}.$

Given a set of data $\{X_1,X_2,\ldots,X_t\}$ up to time $t$, the centralized maximum likelihood estimation problem is given by
\begin{equation}
  \minimize_{\Omega_t \in \mathcal{S}_{++}^p} -\log\det\Omega_t + \trace(S_t\Omega_t)+ \lambda|\Omega_t|_{l1},
  \label{eqn: GL}
\end{equation}
where the set of $p \times p$ positive definite matrices is denoted by set $\mathcal{S}_{++}^p$, and $S_t = \frac{1}{t}\sum_{i=1}^t X_i X_i^T$ is the sample covariance matrix constructed from all of the data up to time $t$. The terms $-\log\det\Omega_t + \trace(S_t\Omega_t)$ are derived from the Gaussian log-likelihood function \cite{yuan2007model}, where $\trace$ denotes trace, and the term $|\Omega_t|_{l1} = \sum_{i,j=1}^p |\Omega_t(i,j)|$ is the element-wise $l_1$ norm, encouraging sparsity of the solution regulated by the penalty parameter $\lambda \geq 0$. 

For the distributed setting, no agent will have access to the entire sample covariance matrix, since each agent only observes a subset of the variables. Instead, we will have each agent $i \in \mathcal{V}_a$ maintain an estimate $S_{i,t}$ of the sample covariance matrix, which it updates over time based on its own measurements and information received from its neighbors. At each time step $t \in \{1,2,\ldots,\}$, the optimization objective for agent $i \in \mathcal{V}_a$ would then be to find an estimate $\Omega_{i,t}$ of $(S^*)^{-1}$ by solving
\begin{equation}
  \minimize_{\Omega_{i,t} \in \mathcal{S}_{++}^p} -\log\det\Omega_{i,t} + \trace(S_{i,t}\Omega_{i,t})+ \lambda|\Omega_{i,t}|_{l1},
  \label{eqn: GL dist}
\end{equation}
where we have simply updated \eqref{eqn: GL} to show the explicit dependence of the estimated sample covariance matrix $S_{i,t}$ of agent $i$ on the current time $t$.

This paper aims to formulate a \textit{distributed} online algorithm for each agent $i$ to calculate and update an estimate of the inverse covariance matrix $\Omega^*=(S^*)^{-1}$ in real-time. Given a sequence of observations of agent $i$, denoted by $\{X_{i,1},X_{i,2},X_{i,3},\ldots\}$, the objective of $i$ is to perform an online inference of the entire edge set $\mathcal{E}$, through estimating the inverse covariance matrix $\Omega_{i,t}$ based on its local data and by incorporating information from its neighbors. In particular, we want all agents to reach consensus on their estimates asymptotically, i.e., as $t \to \infty$, the estimate of each agent converges asymptotically to the agreement: 
$\lim_{t \to \infty}\Omega_{1,t} = \Omega_{2,t} = \ldots = \Omega_{n,t} = \Omega^*.$

In the rest of the paper, we describe a distributed algorithm to solve this problem, allowing the system designer to specify the number of communications and optimization steps of the algorithm between the arrival of data points. We start by describing the online alternating minimization algorithm for solving problem \eqref{eqn: GL} and subsequently describe the extension of the algorithm to the distributed setting.

\section{Background} \label{sec: backgroud}
We build our approach on an alternating minimization algorithm proposed in \cite{dalal2017sparse} (batch) and \cite{OGAMA} (online) for solving problem \eqref{eqn: GL}; the batch algorithm was shown in \cite{dalal2017sparse} to be significantly faster than other proximal methods such as GISTA \cite{rolfs2012iterative} and QUIC\cite{hsieh2011sparse}; the online algorithm was shown in \cite{OGAMA} to achieve a similar result with fewer iterations in real-time settings. We first describe the details of the centralized algorithm from \cite{dalal2017sparse, OGAMA}, and subsequently discuss our modification to account for the distributed estimation.

The approach in \cite{dalal2017sparse} formulates the primal and dual objective functions for problem \eqref{eqn: GL}. The primal of \eqref{eqn: GL} is:
\begin{align} \label{eqn: primal}
\begin{split}
\minimize_{\Omega_t \in \mathcal{S}^p_{++}, \Phi_t \in \mathcal{S}^p_{++}} &-\log\det\Omega_t + \trace(S_t\Phi_t) + \lambda|\Phi_t|_{l1}\\
\st &\Phi_t = \Omega_t.    
\end{split}
\end{align}
The dual of \eqref{eqn: GL} is given by
\begin{align} \label{eqn: dual}
\begin{split}
    \minimize_{\Gamma_t\in\mathcal{S}_{++}^p} &-\log\det\Gamma_t - p\\
    \st &|\Gamma_{t}(i,j) - S_{t}(i,j)| \leq \lambda \quad \forall i,j ,
\end{split}
\end{align}
where the symmetric positive definite matrix $\Gamma_t$ is the dual variable, and $A(i,j)$ denotes the $(i,j)$-th element of matrix $A$.

Given the sample covariance matrix $S_t$, the alternating minimization follows the iterative sequence of updates, where each iteration is indexed by the variable $k \in \mathbb{N}$:
\begin{align} 
    \Omega^{k+1}_t & = \argmin_{\Omega \in\mathcal{S}^p_{++}} -\log\det \Omega + \trace(\Gamma_t^k\Omega)\label{eqn: omega update},\\
    \Phi^{k+1}_t &= \argmin_{\Phi\in\mathcal{S}^p_{++}} \trace(S_t\Phi) + \lambda|\Phi|_{l1}-\trace(\Gamma^k_t\Phi)  \nonumber \\
    &\qquad\qquad\qquad\qquad +\frac{\zeta_t^k}{2}\|\Omega_t^{k+1} - \Phi\|_F^2 \label{eqn: phi update},\\
    \Gamma^{k+1}_t &= \Gamma_t^{k}+\zeta_t^k(\Omega_t^{k+1} - \Phi_t^{k+1}) \label{eqn: gamma update}.
\end{align}
In the above equations,  $\zeta_t^k$ is a step size, and $\|A\|_F = \sqrt{\trace(AA^T)}$ denotes the Frobenius norm of a given matrix $A$. Taking the derivatives of the expressions for $\Omega^{k+1}$ and $\Phi^{k+1}$ and equating them to 0,  we obtain the closed-form updates \cite{dalal2017sparse}:
\begin{align} 
    \Omega^{k+1}_t &= (\Gamma^{k}_t)^{-1},\label{eqn: omega}\\
    \Phi^{k+1}_t &= \frac{1}{\zeta_t^k}\mathcal{S}_\lambda(\zeta_t^k\Omega^{k+1}_t - S_t + \Gamma^{k}_t) \label{eqn: phi}.
\end{align}
Here, $\mathcal{S}_\lambda(x) = \sign(x)(\max(|x|-\lambda,0))$ is the soft-thresholding operator (applied element-wise to a matrix argument). Following these update rules, $\Omega_t$ is interpreted as an approximately sparse inverse covariance matrix, and $\Phi_t$ is interpreted as the estimate of the sparse inverse covariance matrix. Substituting \eqref{eqn: omega} -- \eqref{eqn: phi} for the variables in \eqref{eqn: gamma update}, and using the clip function $\mathcal{C}_\lambda(x) = \min(\max(x,-\lambda),\lambda)$ with the property $x = \mathcal{S}_\lambda(x) + \mathcal{C}_\lambda(x)$, the dual update \eqref{eqn: gamma update} can be written as:
\begin{align} \label{eqn: dual update}
    \Gamma^{k+1}_t = \mathcal{C}_\lambda(\Gamma^{k}_t -S_t + \zeta_t^k(\Gamma^{k}_t)^{-1})+S_t.
\end{align}
In \cite{OGAMA}, the authors proposed an online algorithm (OGAMA) that iteratively updates the sample covariance matrix at each time step $t$ (at the arrival of data points $X_t$) by 
\[S_t = \frac{1}{t}\left(\left(t-1\right)S_{t-1}+X_tX_t^T\right).\] 
The algorithm allows the system designer to select the total number of optimization iterations $K \in \mathbb{N}$ per new data point, and initializes $\Gamma_{t_0}^K = S_{t_0} + \lambda I_p$, where $I_p \in \mathbb{R}^{p \times p}$ is an identity matrix, with data up to a user-defined $t_0 \in \{1,2,\ldots\}$. The algorithm updates the estimates by iterating through \eqref{eqn: dual update} $K$ times at each $t > t_0$, and updates \eqref{eqn: omega} and \eqref{eqn: phi} at the end of each $t$ when $k = K$.

The step size $\zeta_t^k$ at each iteration $k$ is chosen to guarantee convergence of the estimates $\Gamma_t^k$ to their desired quantities. It was shown in \cite{OGAMA} that $\forall t\geq t_0$ and $\forall k \in \{1,2,\ldots,K\}$, the step size can be set as a constant $\zeta = \zeta_t^k < a^2$ for some constant $a$. 

We provide the pseudo-code implementation of the OGAMA algorithm in Algorithm \ref{algo: OGAMA}.
Note that we modified the variables to demonstrate the computation at each agent $i$. In the centralized OGAMA, the input $S_{i,t}^w = S_t$ and the output $\Gamma_{i,t}^k = \Gamma_{t}^k$, $\forall i\in\mathcal{V}_a$ and $\forall w \in \{1,2,\ldots,W\}$, where $w$ denotes the communication iterations which we discuss in the next section.
\begin{algorithm}
\SetKwInOut{Parameter}{Parameter}
\KwIn{$t, k,S_{i,t}^w, \Gamma_{i,t}^{k-1}, \zeta_{i,t}^{k-1}$}
\KwResult{$\Gamma_{i,t}^k, \zeta_{i,t}^{k}$}
\Parameter{$K$, $\lambda$, $t_0$}
\BlankLine
\If{$t=t_0$}{
$\Gamma_{i,t}^k = \Gamma_{i,t_0}^K = S_{i,t_0}^w+\lambda I_p$\\
}
\ElseIf{$t>t_0$}{
$\Gamma^{k}_{i,t} = \mathcal{C}_\lambda(\Gamma^{k-1}_{i,t}-S_{i,t}^w+\zeta_{i,t}^{k-1}(\Gamma_{i,t}^{k-1})^{-1})+S_{i,t}^w$
}
Choose $\zeta_{i,t}^k \in (0, (\lambda_{\min}(\Gamma_{i,t}^k))^2)$
\caption{Online Graphical Alternating Minimization Algorithm (OGAMA)}
\label{algo: OGAMA}
\end{algorithm}

\section{Communication Protocols and Distributed Algorithm} \label{sec: proposed algorithm}
In this section, we first discuss our assumptions on the problem. We then propose the communication protocols in Algorithm \ref{algo: COM} and the high level flow of function executions for DGAMA in Algorithm \ref{algo: DGAMA}.

\begin{assumption}\label{aspt: know neighbor}
Each agent $i \in \mathcal{V}_a$ knows its set of measured variables $X_i$, represented by a diagonal matrix $V_i \in \mathbb{R}^{p \times p}$. The $j$-th diagonal of $V_i(j,j) = 1$ if $x_j \in X_i$ and $V_{i}(j,j) = 0$ otherwise. Also, each agent $i$ knows its neighbors ${N_i}$ and the nodes measured by its neighbors $V_{N_i}$. 
\end{assumption}
\begin{assumption}\label{aspt: sync clock}
The agents have a synchronized clock such that all the agents have the same value of $t$.
\end{assumption}

To tailor the algorithm for the real-time multi-agent setting, we allow the system designer to select the maximum number of communication iterations $W$ and the maximum number of optimization iterations $K$ per time-step. We discuss the range of $W$ and $K$ in Sec. \ref{sec: analysis}. The selection of $W$ and $K$ is also based on the number of variables $p$, the hardware capabilities, the data arrival rate, etc. When communication (or computation) is expensive or slow, $W$ (or $K$) is set to be small; on the other hand, if communication (or computation) is cheap or fast, $W$ (or $K$) can be set to be large. We let $W$ and $K$ remain the same for each time step for simplicity. We index each communication iteration by the variable $w \in \{1,2,\ldots,W\}$ and index the optimization iteration by the variable $k \in \{1,2,\ldots,K\}$. In particular, we denote agent $i$'s estimate of $S_t$ after $w$ rounds of communication at time-step $t$ by $S_{i,t}^w$.

We initialize each agent's estimated sample covariance matrix as $S_{i, 0} = \textbf{0} \in \mathbb{R}^{p\times p}$. At each time step $t$, the agents communicate with their neighbors to update $S_{i,t}^w$. The pseudo-code implementation to communicate and update $S_{i,t}^w$ is in Algorithm \ref{algo: COM}. 

At the beginning of time $t$, when $w = 1$, upon the arrival of the observation data $X_{i,t}$, each agent $i$ performs a round of communication by sending its $X_{i,t}$ and receiving $X_{N_i,t}$ from its neighbors. Based on this communication data, we introduce the concept of an \textit{observable} node in the following definition. 
\begin{definition}
We say node $x_j$ is \textit{observable} by agent $i$ if $x_j \in  X_{\mathcal{N}_i} = X_{i}\cup X_{N_i}$, i.e., the node can be directly observed by agent $i$ or indirectly observed through data from ${N_i}$. Similarly, we say the $(l,m)$-th entry of sample covariance matrix $S_{i}(l,m)$ is \textit{observable} by agent $i$ if $x_l$ and $x_m$ are both observable by agent $i$.
\end{definition}

Observations of the pairwise relationships are required to solve our inverse covariance estimation problem; thus, each pair of nodes must be observed by at least one agent. With that in mind, we proposed the following definition. 
\begin{definition}\label{def: jointobserv}
The pairwise relationships of all nodes (i.e., random variables) are \textit{jointly observable} if \[\cup_i \{X_{\mathcal{N}_i} \times X_{\mathcal{N}_i} \} = X \times X.\]
\end{definition} 
Joint observability ensures that for each pair of nodes, there is at least one agent that can observe and compute the sample covariance between these nodes. A demonstration of Def. \ref{def: jointobserv} is presented in Fig. \ref{fig: demo}.
\begin{figure}
    \centering
    \includegraphics[width = 0.8\linewidth]{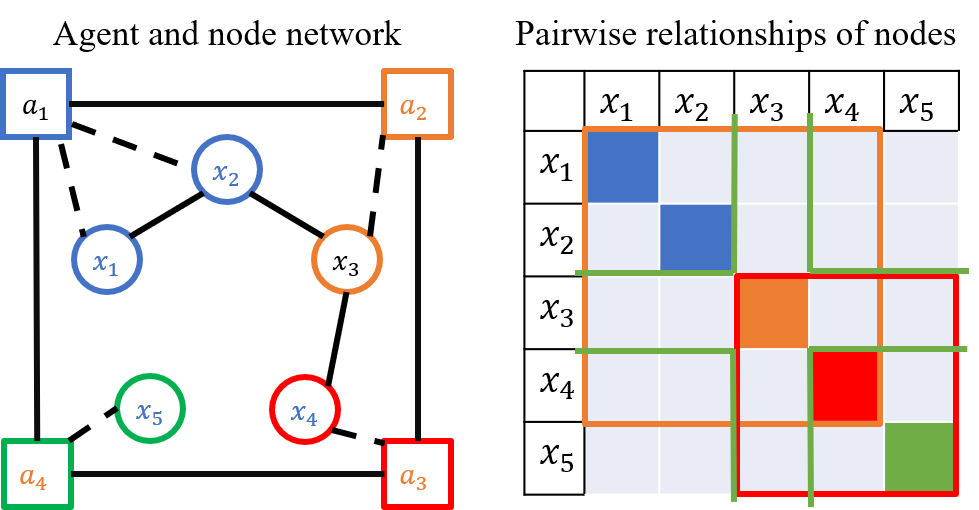}
    \caption{Four agents $(a_1, \ldots, a_4)$ observing five nodes $(x_1, \ldots, x_5)$. The solid lines connecting the agents indicate communication; the solid lines connecting the nodes indicate correlations; the dashed lines between the agents and the nodes indicate that the agents directly observe nodes. In this example, the sample covariance matrix is jointly observable, i.e., for each pair of nodes, there is an agent that can observe the pairwise relationship. The entries of the sample covariance matrix observed by agent $1$ are omitted in the presentation.}
    \label{fig: demo}
\end{figure}

After the communications among the agents, each agent computes the sample covariance matrix $S_{i,t}^w$.
If $w= 1$, with data $X_{i,t}$ and $X_{N_i,t}$, each agent creates a data vector $\chi_{i,t} \in \mathbb{R}^{p}$. If a variable $x_{j}$ is observable by agent $i$ at time $t$, the corresponding entry $j$ of $\chi_{i,t}(j) =  x_{jt}$, and $\chi_{i,t}(j) = 0$ otherwise. With $\chi_{i,t}$, each agent computes $S_{i,t}^1$ with the following updates:
\begin{align}
        S_{il,t} &= \frac{1}{t}\left(\left(t-1\right)V_iS_{i,t-1}^W V_i+\chi_{i,t}\chi_{i,t}^T\right)
        \label{eqn: local}\\
        S_{i,t}^1 &= S_{il,t} + \frac{\sum_{j\in \mathcal{N}_i} S_{j,t-1}^{W} -V_i \sum_{j\in \mathcal{N}_i} S_{j,t-1}^{W} V_i}{|\mathcal{N}_i|}.
        \label{eqn: updateS}
\end{align}

Each agent updates its local sample covariance matrix, denoted by $S_{il,t}$, computed recursively following \eqref{eqn: local}, consisting of data up to time step $t$.
The outer product $\chi_{i,t}\chi_{i,t}^T$ creates a block matrix, such that, if $x_l$ and $x_m$ are both observable by agent $i$, the $(l,m)$-th entry $ \chi_{i,t}\chi_{i,t}^T(l,m) = x_{lt}x_{mt}$, and 0 otherwise. Similarly, the term $V_i S_{i,t-1}^{W} V_i$ creates a block matrix such that $ V_i S_{i,t-1}^{W} V_i(l,m) = S_{i,t-1}^{W}(l,m)$ if $x_l$ and $x_m$ are both observable by $i$, and 0 otherwise. 
In \eqref{eqn: updateS}, the sample covariance matrix $S_{i,t}^1$ of $i$ is created by combining the observable entries from the block matrix $S_{il,t}$, with the unobservable entries from $\sum_{j \in \mathcal{N}_i} S_{j,t-1}^{W}/|\mathcal{N}_i|$.
If the entries of $S_{i,t}^{w}$ are observable by $i$, agent $i$ uses its information exclusively; if the entries of $S_{i,t}^{w}$ are not observable by $i$, agent $i$ uses the information from its neighbors. Agent $i$ then sends its neighbors the new $S_{i,t}^{1}$ and receives $S_{N_i,t}^1$ from its neighbors.

When $w > 1$, the update of $S_{i,t}^w$ follows
\begin{multline}
    S_{i,t}^w = V_iS_{i,t}^{w-1}V_i + \frac{\sum_{j\in \mathcal{N}_i} S_{j,t}^{w-1} - V_i \sum_{j\in \mathcal{N}_i} S_{j,t}^{w-1} V_i}{|\mathcal{N}_i|}, 
    \label{eqn: updateS w>1}
\end{multline}
where $V_iS_{i,t}^{w-1}V_i$ contains the observable entries of agent $i$ from the previous step (0 otherwise). Then agent $i$ computes $S_{i,t}^w$ by integrating the observable entries from the previous iteration $S_{i,t}^{w-1}$ and unobservable entries given by $\sum_{j\in \mathcal{N}_i} S_{j,t}^{w-1}/|\mathcal{N}_i|$. Agent $i$ then sends its neighbors the new $S_{i,t}^{w}$ and receives $S_{N_i,t}^w$ from its neighbors.

\begin{algorithm}
\SetKwInOut{Parameter}{Parameter}
\KwIn{$t, w, S_{\mathcal{N}_i,t}^{w-1}$}
\KwResult{$S_{\mathcal{N}_i,t}^{w} = S_{i,t}^w \cup S_{N_i,t}^w$}
\Parameter{$V_i, N_i$}
\BlankLine
\If{$w = 1$}{
    Send data $X_{i,t}$ and receive neighbor data $X_{N_i,t}$ \\
    $X_{\mathcal{N}_i,t} = X_{i,t} \cup X_{N_i,t}$\\
    Create $\chi_{i,t}$ with $X_{\mathcal{N}_i,t}$ as described above \eqref{eqn: local}\\
    Update $S_{il,t}$ as per \eqref{eqn: local}\\
    Update $S_{i,t}^1$ as per \eqref{eqn: updateS}
}
\Else{
Update $S_{i,t}^w$ as per \eqref{eqn: updateS w>1}
}
Send $S_{i,t}^w$ and receive $S_{N_i,t}^w$ 
\caption{Communication for Agent $i \in \mathcal{V}_a$ (COM)}
\label{algo: COM}
\end{algorithm}

In Algorithm \ref{algo: DGAMA}, we present the flow of function executions. Each agent $i$ performs $W$ rounds of communication and obtains $S_{i,t}^W$ at each time step $t$; subsequently, $i$ updates its local estimates $\Gamma_{i,t}^K$ after performing $K$ optimization iterations.

\begin{algorithm}
\KwIn{$W, K, \lambda, t_0$}
Initialize $S_{j,1}^0 = S_{j,0}^W = \textbf{0} \in \mathbb{R}^{p \times p} \; \forall j \in \mathcal{N}_i$\\

\For{$t\in\{1,2,3,\ldots\}$}{
    \For{$w = 1:W$}{
    
        $S_{\mathcal{N}_i,t}^w =$ COM($t, w, S_{\mathcal{N}_i,t}^{w-1}$) $\;$ (consensus) 
    }
    Extract $S_{i,t}^W$ from $S_{\mathcal{N}_i,t}^W$\\
    \If{$t = t_0$}{$\Gamma_{i,t}^K, \zeta_{i,t}^{K}$ = OGAMA($t, K, S_{i,t}^W, null, null$)}
    \If{$t > t_0$}{
    \For{$k = 1:K$}{
    \If{$k = 1$}{$\Gamma_{i,t}^0 = \Gamma_{i,t-1}^K$}
        $\Gamma_{i,t}^k, \zeta_{i,t}^{k}$ = OGAMA($t, k,S_{i,t}^W, \Gamma_{i,t}^{k-1}, \zeta_{i,t}^{k-1}$)
    }
    Update the variables $\Omega_t^K$, $\Phi_t^K$ as per \eqref{eqn: omega}, \eqref{eqn: phi}
    }
}
\caption{Distributed Graphical Alternating Minimization Algorithm for Agent $i \in \mathcal{V}_a$ (DGAMA)}
\label{algo: DGAMA}
\end{algorithm}

\section{Convergence Analysis} \label{sec: analysis}
In this section, we show that the dual variable estimate $\Gamma_{i,t}^K$ of each agent $i$ converges to $\Gamma^*$, where $\Gamma^* = (\Omega^*)^{-1}$ is the centralized fixed point solution of \eqref{eqn: dual update} and $\Omega^*$ is the centralized optimal solution of \eqref{eqn: GL}, given the ground truth covariance matrix $S^*$ \cite{dalal2017sparse}.
The proofs of our results can be found in the Appendix. 

\subsection{Error Bounds of Dual Variables}

To simplify the analysis, we assume that $\forall i \in \mathcal{V}_a, t \geq t_0$, and $k\in\{1,2,\ldots,K\}$, we can select a constant step size $\zeta = \zeta_{i,t}^k <a^2$, such that $0\prec aI_p \preceq \Gamma_{i,t}^k \preceq b$ and $aI_p \preceq \Gamma^* \preceq bI_p$. In Appendix \ref{subsec: theorem 1}, we show that constants $a$ and $b$ exist.

We provide the following results bounding the difference between the estimate $\Gamma_{i,t}^K$ of each agent and the centralized optimal solution $\Gamma^*$, after $K$ iterations at each time step $t$ (given the new data point $X_{i,t}$), denoted by $\|\Gamma_{i,t}^K - \Gamma^*\|_F.$

\begin{theorem}
Assume $\forall i \in \mathcal{V}_a$, $\forall t \geq t_0$, and $\forall k \in \{1, 2, \ldots, K\}$, the iterates $\Gamma_{i,t}^k$ satisfy $a I_p\preceq \Gamma_{i,t}^k \preceq b I_p$,  and $aI_p \preceq \Gamma^* \preceq bI_p$, for some fixed constants $0<a<b$, and $\zeta_{i,t}^k = \zeta$.
After $W$ communication and $K$ optimization iterations of Algorithm \ref{algo: DGAMA}, at time step $t$, we have the following bound for the dual variable:
\begin{multline}\label{eqn: algo 3 bound}
    \|\Gamma_{i,t}^K - \Gamma^*\|_F \leq \beta^{K(t-t_0)} \|\Gamma_{i,t_0}^K - \Gamma^*\|_F\\
    +2\sum_{m=1}^K \beta^{K-m}\sum_{l=t_0+1}^{t} \beta^{K(t-l)}\|S_{i,l}^W-S^*\|_F,
\end{multline}where $\beta = \max\left\{|1-\frac{\zeta}{a^2}|,|1-\frac{\zeta}{b^2}|\right\}.$
\label{thm: conv of 1}
\end{theorem}
\subsection{Convergence of Sample Covariance Matrix}
To analyze the convergence of the estimates $\Gamma_{i,t}^k$ of the algorithms, we show that the sample covariance matrix $S_{i,t}^w$ of each agent converges to the true sample covariance matrix $S^*$. 
We consider each entry $(l,m)$ of the sample covariance matrix $S_t$. We let $L \subset \mathcal{V}_a$ denote the subset of agents that are observing $S_t(l,m)$ and call these agents \textit{leaders}. The other agents who cannot observe $S_t(l,m)$ are called \textit{followers} and denoted by $F =  \mathcal{V}_a\setminus L$. For our discussion below, we focus on a specific pair of $(l, m)$, and without loss of generality, we let the first $|L|$ agents $\{1,2,\ldots,{|L|}\}$ be leaders and let the last $|F|$ agents $\{{|L|+1}, {|L|+2},\ldots,{n}\}$ be followers. The same analysis holds for all $l, m$.

For entry $(l,m)$, if $i \in L$, then $S_{i,t}^w(l,m) = S_{t}(l,m), \forall t$ and $\forall w = \{1,2,\ldots,W\}$, where $S_{t}$ is the sample covariance matrix computed given data $\{X_1,X_2,\ldots, X_t\}$ up to time $t$. If $S_t(l,m)$ is not observed by agent $i$, the update of unobserved entries in \eqref{eqn: updateS} and \eqref{eqn: updateS w>1} can be written as
\begin{equation}
S_{i,t}^w (l,m)= \frac{1}{|\mathcal{N}_i|}\sum_{j\in \mathcal{N}_i} S_{j,t}^{w-1}(l,m) \; \forall i\in F.    
\end{equation}
We can write this update for the follower agents in the matrix form,
\begin{align*}
S_{F,t}^w(l,m) &= \Tilde{D}^{-1}(\Tilde{I}+\Tilde{A})S_{t}^{w-1}(l,m) \\
&\triangleq P_{F,lm} S_{t}^{w-1}(l,m),
\end{align*}
where matrix $\Tilde{D} = \diag(|\mathcal{N}_{|L|+1}|, |\mathcal{N}_{|L|+2}|,\ldots,|\mathcal{N}_{n}|)$, $\Tilde{I}$ is the last $|F|$ rows of identity matrix $I_n$, $\Tilde{A}$ is the last $|F|$ rows of an unweighted adjacency matrix, and $S_{t}^w(l,m) = [S_{1,t}^w(l,m),S_{2,t}^w(l,m),\ldots,S_{n,t}^w(l,m)]^T$.

Thus, for each $(l,m)$ entry of $S_t$, the updates of Algorithm \ref{algo: COM} can be written in the form
\[\begin{bmatrix} S_{L,t}^w(l,m) \\ S_{F,t}^w(l,m) \end{bmatrix} = 
\begin{bmatrix} P_{LL,lm} & P_{LF,lm} \\ P_{FL,lm} & P_{FF,lm}\end{bmatrix}
\begin{bmatrix} S_{L,t}^{w-1}(l,m) \\ S_{F,t}^{w-1}(l,m)\end{bmatrix}.\] 
Since the leader agents keep their estimations constant, $P_{LL,lm} = I_{|L|}$ and $P_{LF,lm} = \textbf{0}$. 
The matrix $P_{FL,lm}$ is the first $|L|$ columns and $P_{FF,lm}$ is the last $|F|$ columns of $P_{F,lm}$. It was shown that the convergence rate of the communication network is dictated by the largest eigenvalue of $P_{FF,lm}$ \cite{smalleigstubborn}. We provide the following lemma using $P_{FF,lm}$.

\begin{lemma} If all variables are jointly observable, the sample covariance matrix of each agent $S_{i,t}^w$ converges to $S_t$ as $w \to \infty$, with a rate given by
\begin{equation} 
    \|S_{i,t}^w - S_t\|_F \leq c\sigma^w \|S_{i,t}^0-S_t\|_F,
    \label{eqn: S consensus}
\end{equation}where $\sigma = \max_{lm}\lambda_{\max}(P_{FF,lm})$ denotes the largest eigenvalue of $P_{FF,lm}$ for all $(l,m)$, $c \geq 1$ is some constant related to the degrees of $\mathcal{G}_a$ \cite{stubbornagent}, and $S_{i,t}^0 = S_{i,t-1}^w$ denotes $i$'s initial estimate of the sample covariance matrix at time $t$. 
\label{lemma: consensus}
\end{lemma}

Lemma \ref{lemma: consensus} shows that the states of all agents converge at a rate given by the largest eigenvalue of $P_{FF,lm}$. In this work, we will omit the discussion on the bound of $\sigma$ and refer the readers to \cite{opinion_stubborn, stubbornagent, smalleigstubborn}. Since $0\leq \sigma <1$ ($\mathcal{G}_a$ is connected), all agents reach consensus when the rounds of communication $W \to \infty$, i.e., $\lim_{W\to\infty}S_{i,t}^W = S_{t}, \forall i \in \mathcal{V}_a$ . To guarantee the convergence of DGAMA, we let $W$ be large enough such that each agent's estimate of $S_t$ after $W$ rounds of communication $S_{i,t}^W$ is better than the initial estimate $S_{i,t}^0$, i.e., $c\sigma^W < 1$. Next, we show that all agents reach consensus for such $W$ as the number of time steps $t \to \infty$. 
\begin{theorem}
\label{thm: S conv}
If all variables are jointly observable, the estimates $S_{i,t}^W$ of sample covariance matrices from Algorithm \ref{algo: COM} converge to the ground truth sample covariance matrix $S^*$ almost surely as $t \to \infty$, for all $i \in  \mathcal{V}_a$ and for all $W \in \mathbb{N}$ such that $c\sigma^W < 1$: \[\lim_{t\to \infty} \|S_{i,t}^W-S^*\|_F = 0 \; a.s.\]
\end{theorem}

\subsection{Convergence of Dual Variables}
In this subsection, we show the convergence and the rate of convergence of the dual variables $\Gamma_{i,t}^K$ as $t \to \infty$.

\begin{corollary}
\label{coro: 1}
Assume there exist constants $0 < a < b$ such that, $\forall t\ge t_0$, $\forall k \in \{1, 2, \ldots, K\}$, and $\forall i \in  \mathcal{V}_a$, the quantities $\Gamma_{i,t}^k$ satisfy $aI\preceq \Gamma_{i,t}^k \preceq bI$, $aI_p \preceq \Gamma^* \preceq bI_p$, and $\zeta_{i,t}^k = \zeta < a^2$. Then, for all $W \in \mathbb{N}$ such that $c\sigma^W < 1$, as the number of data points $t \to \infty$, the result $\Gamma_{i,t}^K$ converges to the optimal solution $\Gamma^*$ almost surely: \[\lim_{t\to \infty}\|\Gamma_{i,t}^K - \Gamma^*\| = 0 \; a.s.\]
\end{corollary}

Having established that the iterates provided by DGAMA converge, we characterize the asymptotic rate of convergence for the dual variable of each agent $\Gamma_{i,t}^K$ to $\Gamma^*$. 

\begin{corollary}
\label{coro: rateofconvergence}
Assume $\forall i \in \mathcal{V}_a$, $\forall t \geq t_0$ and $\forall k \in \{1,2,\ldots,K\}$, the iterates $\Gamma_{i,t}^k$ satisfy $0 \prec aI_p \preceq \Gamma_{i,t}^k \preceq bI_p$, $\Gamma^*$ satisfies $aI_p \preceq \Gamma^* \preceq bI_p$, and $\zeta_{i,t}^k = \zeta$. Then for all $\Delta \in (0,\frac{1}{2})$, for all sample paths in a set of measure 1, there exists a $\bar{t}$, such that $\forall t \geq \bar{t} + 1$, the update of dual variable $\Gamma_{i,t}^K$ satisfies the following condition:
\begin{multline}
\|\Gamma_{i,t}^{K}-\Gamma^*\|_F  
\leq \beta^{K(t+1-\bar{t})}\|\Gamma_{\bar{t}}^K-\Gamma^*\|_F \\+ 2\sum_{m=1}^K\beta^{K-m}\sum_{l = \bar{t}+1}^t \beta^{K(t-l)}\Big[(c\sigma^W)^{l-\bar{t}}\|S_{i,\bar{t}}^W - S_{\bar{t}}^*\|_F \\
+ 40p\max_j\left(S^*\left(j,j\right)\right)\\ \cdot \big((\frac{1}{l})^{\frac{1}{2}-\Delta}+ 2\sum_{j = \bar{t}+1}^{l} (c\sigma^W)^{l+1-j}(\frac{1}{j-1})^{\frac{1}{2}-\Delta}\big)\Big],
\label{eqn: rate of conv}
\end{multline}
where $\beta = \max\left\{|1-\frac{\zeta}{a^2}|,|1-\frac{\zeta}{b^2}|\right\}$.
\end{corollary} 

Note from the above result that when $0 < \beta < 1$, the asymptotic rate at which $\Gamma_{i,t}^K$ decreases to zero is dominated by the second term in \eqref{eqn: rate of conv}. Specifically, the rate is dominated by the convergence of the sample covariance matrix. Also, increasing $K$ will speed up the convergence of the estimates $\Gamma_{i,t}^K$. Comparing to the rate of convergence in the centralized \cite{OGAMA}, increasing $W$ will decrease the deviation of each agent's estimates to the centralized estimates.
We illustrate the dependence of this bound on $t$, $K$, and $W$ in the next section with experiments. 

\begin{figure*}
    \centering
    \begin{subfigure}[b]{0.3\textwidth}
         \centering
         \includegraphics[width=\textwidth]{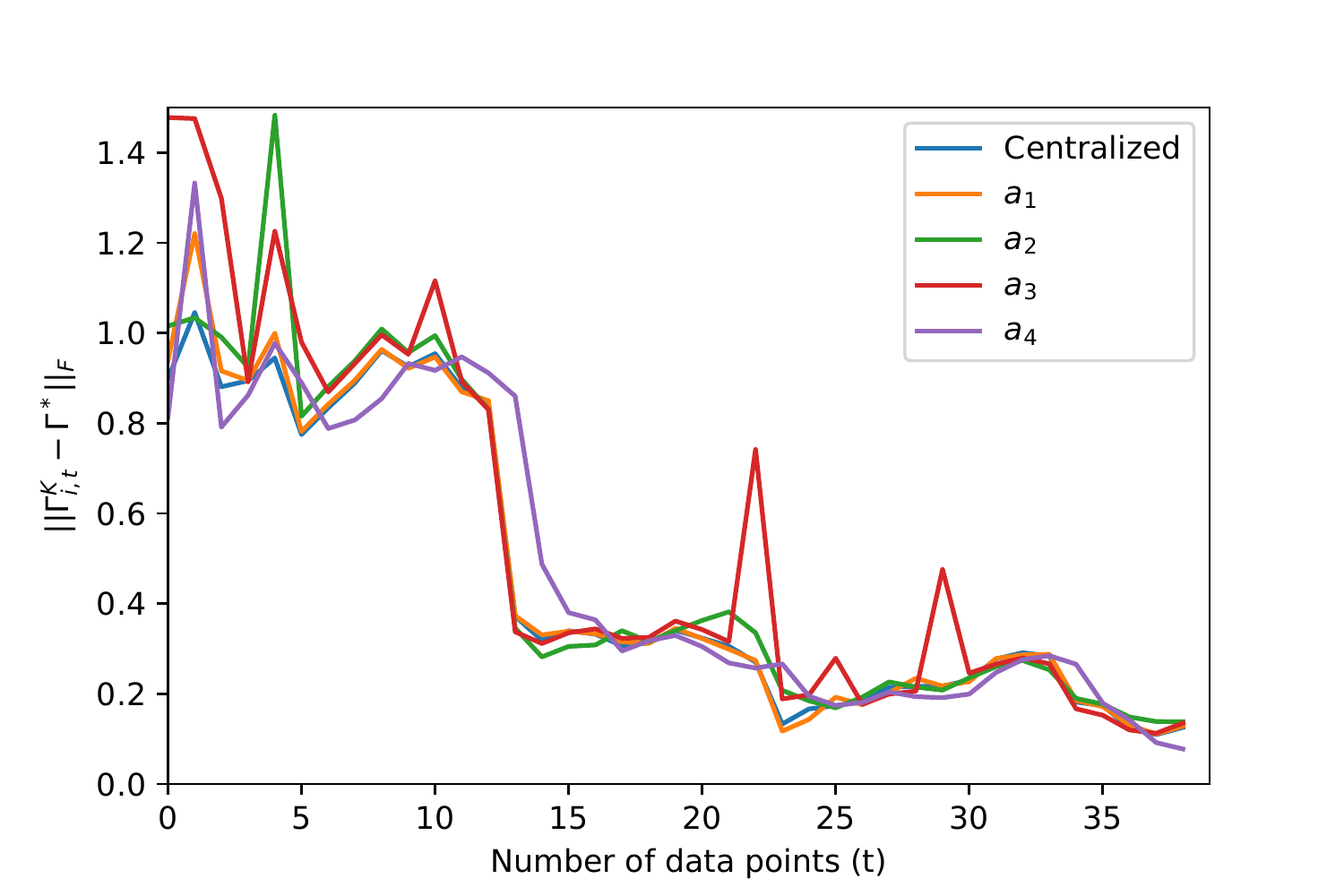}
         \caption{$K = 1, W = 1$}
    \end{subfigure}
    \begin{subfigure}[b]{0.3\textwidth}
         \centering
         \includegraphics[width=\textwidth]{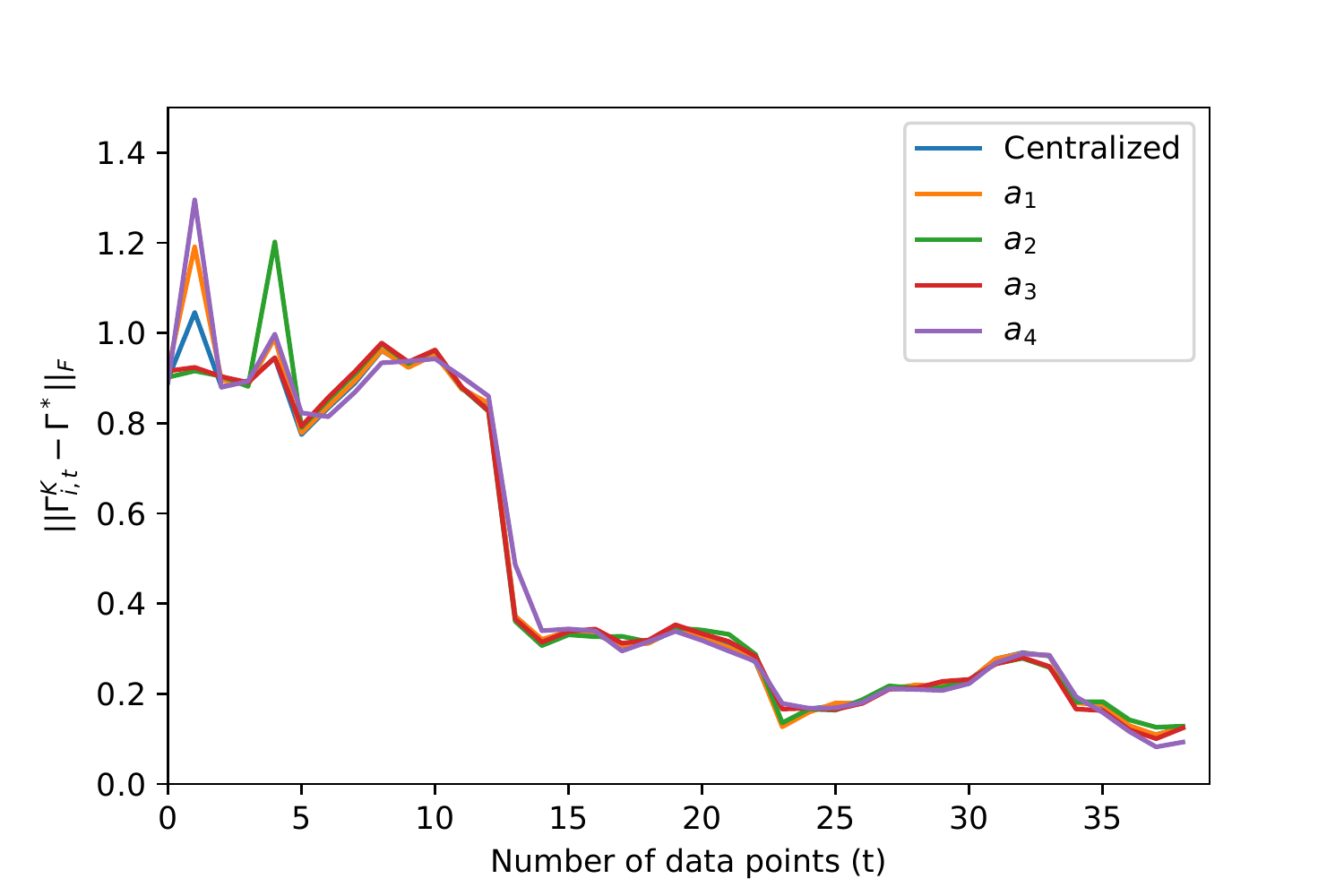}
         \caption{$K = 1, W = 2$}
     \end{subfigure}
        \begin{subfigure}[b]{0.3\textwidth}
         \centering
         \includegraphics[width=\textwidth]{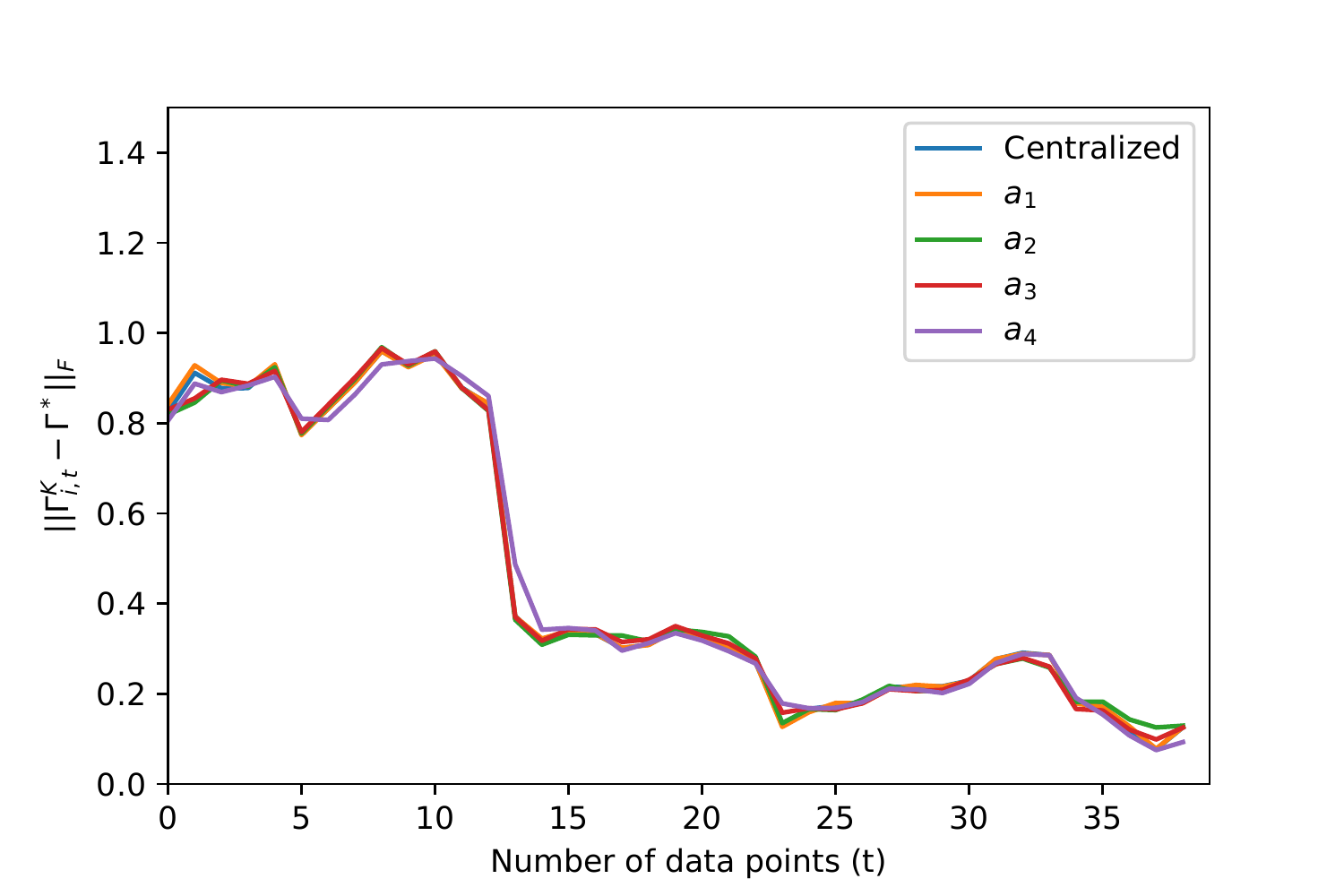}
         \caption{$K = 2, W = 2$}
     \end{subfigure}
     \caption{Simulations with varying parameters.}
     \label{fig: simu}
\end{figure*}
\section{Experiments}
\label{sec: experiments}
We build the agent and node networks shown in Fig. \ref{fig: demo} ($p=5,n=4$). We build the node network by generating a sparse Erdos-Renyi network representing $({S^*})^{-1}$, following the steps in \cite{gengraph}. We generate $40$ independent and identically distributed samples from a Gaussian distribution $\mathcal{N}(0, S^*)$ and run DGAMA Algorithm with $\lambda = 0.15$ and $t_0 = 10$.

In Fig. \ref{fig: simu}, we show the convergence of our algorithm with different values of $K$ (the number of iterations per data point) and $W$ (the number of communication rounds per data point) in comparison with the centralized online algorithm from \cite{OGAMA}. The deviation of the iterates, measured by $\|\Gamma_{i,t+1}^K- \Gamma^*\|_F$, converges to zero as the number of data points $t$ increases. We see that the distributed algorithm provides estimates comparable to that of the centralized online algorithm, with small $K$ and $W$. Moreover, as $W$ increases, the estimates are closer to that of the centralized algorithm; as $K$ increases, the estimates of the algorithm converges to $\Gamma^*$ faster.

\section{Conclusion}\label{sec: conclusion}
In this paper, we proposed a distributed and online sparse inverse covariance matrix estimation algorithm. We provided convergence guarantees of each agent's estimates to what would have been obtained at a centralized agent if all of the data were available simultaneously. Finally, we demonstrated the empirical performance of our algorithm. 

\bibliographystyle{unsrt}
\bibliography{root}

\pagebreak
\appendix
\subsection{Proof of Theorem 1}
\label{subsec: theorem 1}
To prove Theorem 1, we first provide the following lemma showing that the eigenvalues of the estimates are bounded. 
\begin{lemma}
\label{lemma: eigen}
Let $t_0$ be a time such that the sample covariance matrix $S_{i,t_0}^w$ is nonsingular $\forall i \in \mathcal{V}_a$ and $\forall w \in \{1,2,\ldots,W\}$. 
Define $a = \min\{\lambda_{\min}(S^*),\inf_{i, w, t\geq t_0} \lambda_{\min}(S_{i,t}^w)\} - p\lambda$ and $b = \max\{\lambda_{\max}(S^*), \sup_{i, w, t\geq t_0} \lambda_{\max}(S_{i,t}^w)\}+p\lambda$. 
Then $\forall t \geq t_0, \forall k\in \{1,2,\ldots,K\}$, and $\lambda \in (0,\frac{1}{p}\min\{\lambda_{\min}(S^*),\inf_{i, w, t\geq t_0}\lambda_{\min}(S_{i,t}^w)\})$, the iterates $\Gamma_{i,t}^k$ of Algorithm \ref{algo: OGAMA} satisfy $0 \prec aI_p \preceq \Gamma_{i,t}^k \preceq bI_p$ and $aI_p \preceq \Gamma^*\preceq bI_p$.
\end{lemma}
\begin{proof}
The proof is similar to \cite{OGAMA}, where we simply consider the infimum and supremum of the eigenvalues of $S_{i,t}^w$, $\forall i\in \mathcal{V}_a, \forall w\in\{1,2,\ldots,W\}$, and $\forall t \geq t_0$, for $a$ and $b$ respectively.
\end{proof}

The following lemma provides a bound of $\|\Gamma_{i,t}^k -\Gamma^*\|$ using the error from the previous iteration.  

\begin{lemma}
\label{lemma:k}
Assume $\forall i \in \mathcal{V}_a$, $\forall t \geq t_0$, and $\forall k \in \{1, 2, \ldots, K\}$, iterates $\Gamma_{i,t}^k$ satisfy $a I_p\preceq \Gamma_{i,t}^k \preceq b I_p$ and $\Gamma^*$ satisfies $aI_p \preceq \Gamma^* \preceq bI_p$ for some fixed constants $0<a<b$.
At the $k$-th iteration of time-step $t$, we have the following bound for the dual variable:
\[
\|\Gamma_{i,t}^k-\Gamma^*\|_F \leq \beta_{i,t}^{k-1}\|\Gamma_{i,t}^{k-1} - \Gamma^*\|_F + 2\|S_{i,t}^W-S^*\|_F,
\] where $\beta_{i,t}^{k-1} = \max\left\{|1-\frac{\zeta_{i,t}^{k-1}}{a^2}|,|1-\frac{\zeta_{i,t}^{k-1}}{b^2}|\right\}$.
\end{lemma}
\begin{proof}
 Assume $aI_p \preceq \Gamma_t^k \preceq bI_p , \forall t \geq t_0$, $\forall k \in \{1,2,\dots,K\}$ and $aI_p \preceq \Gamma^* \preceq bI_p$. Using the dual update \eqref{eqn: dual update}, matrix norm triangle inequality, and the non-expansive property of the clip function $\mathcal{C}_\lambda$\cite{dalal2017sparse}, we have:
\begin{align*}
    & \|\Gamma_{i,t}^k - \Gamma^*\|_F = \|\mathcal{C}_\lambda(\Gamma_{i,t}^{k-1}+\zeta_{i,t}^{k-1}(\Gamma_{i,t}^{k-1})^{-1}-S_{i,t}^W)+S_{i,t}^W
    \\&\qquad-\mathcal{C}_\lambda(\Gamma^*+\zeta_{i,t}^{k-1}(\Gamma^*)^{-1}-S^*)-S^*\|_F\\
    &\leq \|\mathcal{C}_\lambda(\Gamma_{i,t}^{k-1}+\zeta_{i,t}^{k-1}(\Gamma_{i,t}^{k-1})^{-1}-S_{i,t}^W)
    \\&\qquad -\mathcal{C}_\lambda(\Gamma^*+\zeta_{i,t}^{k-1}(\Gamma^*)^{-1}-S^*)\|_F+\|S_{i,t}^W-S^*\|_F.\\
    &\leq \|(\Gamma_{i,t}^{k-1}+\zeta_{i,t}^{k-1}(\Gamma_{i,t}^{k-1})^{-1}-S_{i,t}^W)\\
    &\qquad-(\Gamma^*+\zeta_{i,t}^{k-1}(\Gamma^*)^{-1}-S^*)\|_F +\|S_{i,t}^W-S^*\|_F\\
    &\leq \|(\Gamma_{i,t}^{k-1}+\zeta_{i,t}^{k-1}(\Gamma_{i,t}^{k-1})^{-1})-(\Gamma^* + \zeta_{i,t}^{k-1}(\Gamma^*)^{-1})\|_F
    \\&\qquad +2\|S_{i,t}^W-S^*\|_F \\
    &\leq \max\left\{\left| 1-\frac{\zeta_{i,t}^{k-1}}{b^2}\right|,\left| 1-\frac{\zeta_{i,t}^{k-1}}{a^2}\right|\right\} \|\Gamma_{i,t}^{k-1}-\Gamma^*\|_F \\
    &\qquad+ 2\|S_{i,t}^W - S^*\|_F,
\end{align*}
where the last inequality follows from Lemma in \cite{rolfs2012iterative}.
\end{proof}

\subsubsection*{Proof of Theorem 1}
Assume $\forall i \in \mathcal{V}_a$, $\forall t \geq t_0$, and $\forall k \in \{1,2,\dots,K\}$, we can select a constant step size $\zeta = \zeta_{i,t}^k$, where $0 \prec aI_p \preceq \Gamma_t^k \preceq bI_p$ and $aI_p \preceq \Gamma^* \preceq bI_p$.
For $t \in \{t_0+1,t_0+2,\ldots\}$, using the results from Lemma \ref{lemma:k},
\begin{align*}
    &\|\Gamma_{i,t_0+1}^1 - \Gamma^*\|_F
    \leq \beta \|\Gamma_{i,t_0}^K - \Gamma^*\|_F+ 2\|S_{i,t_0+1}^W-S^*\|_F \\
    &\|\Gamma_{i,t_0+1}^2 - \Gamma^*\|_F
    \leq \beta^2 \|\Gamma_{i,t_0}^K - \Gamma^*\|_F+2\beta\|S_{i,t_0+1}^W-S^*\|_F\\
    &\qquad\qquad\qquad+2\|S_{i,t_0+1}^W-S^*\|_F.
    \end{align*}
    After $K$ iterations,  
    \begin{multline*}
    \|\Gamma_{i,t_0+1}^K - \Gamma^*\|_F
    \leq \beta^K \|\Gamma_{i,t_0}^K - \Gamma^*\|_F \\+ 2\sum_{m=1}^K \beta^{K-m}\|S_{i,t_0+1}^W-S^*\|_F.     
    \end{multline*}
    At time $t$, at the end of $K$ iterations, we arrive at
    \begin{align*}
    &\|\Gamma_{i,t}^K - \Gamma^*\|_F \leq \beta^{K(t-t_0)} \|\Gamma_{i,t_0}^K - \Gamma^*\|_F\\
    &\qquad \qquad+2\sum_{l=t_0+1}^{t} \beta^{K(t-l)}\sum_{m=1}^K \beta^{K-m}\|S_{i,l}^W-S^*\|_F.
\end{align*}

\subsection{Proof of Lemma 1}
\label{subsec: lemma 1}
Consider any entry $(l,m)$ of $S_t$. Under Def. \ref{def: jointobserv}, there is at least one leader agent observing $S_t(l,m)$. If $i$ observes $S_t(l,m)$, then $S_{i,t}^w(l,m) = S_{t}(l,m)$ $\forall i \in L$ and $\forall w \in \{1,2,\ldots,W\}$. Since graph $\mathcal{G}_a$ is connected, the values of every follower agent will converge to the same value as the leader agents, i.e., $\lim_{w \to \infty}S_{i,t}^w(l,m) = S_{t}(l,m), \forall i \in F,$  geometrically with a rate at least equal to the largest eigenvalue of $P_{FF,lm}$ \cite{stubbornagent, opinion_stubborn}. 

\subsection{Proof of Theorem 2}
\label{subsec: Theorem 2}
To prove Theorem 2, we introduce the following lemma.
\begin{lemma}[\cite{nedic2010constrained}]\label{lemma: nedic}
Let $0<\alpha<1$ and let $\{\gamma_t\}$ be a positive scalar sequence. Assume that $\lim_{t\to\infty}\gamma_t = 0$. Then
\[\lim_{t\to \infty}\sum_{l=0}^{t} \alpha^{t-l}\gamma_l = 0.\]
\end{lemma}
\subsubsection*{Proof of Theorem 2}
First, we characterize an error bound for $S_{i,t}^W$. Using triangle inequality and applying Lemma \ref{lemma: consensus}, we have
\begin{align*}
\|S_{i,1}^W-S^*\|_F 
&\leq \|S_{i,1}^W - S_{1}\|_F + \|S_{1}-S^*\|_F\\
& \leq c\sigma^W \|S_{i,1}^0 - S_{1}\|_F + \|S_{1}-S^*\|_F.
\end{align*}
Define $e_{t} = S_{t} - S_{t-1}, \forall t \geq 1$. Recall that for initialization $S_{i,0}^w =\textbf{0}, \forall i,w$ and $S_0 = \textbf{0}$. As a result, 
\[
\|S_{i,1}^W - S_{1}\|_F \leq c\sigma^W\|S_0 + e_1\|_F = c\sigma^W\|e_1\|_F,\]
and
\[\|S_{i,1}^W-S^*\|_F 
\leq c\sigma^W \|e_1\|_F + \|S_{1}-S^*\|_F.\]
Similarly,
\begin{align*}
\|S_{i,2}^W&-S^*\|_F 
\leq c\sigma^W \|S_{i,2}^0 - S_{2}\|_F + \|S_{2}-S^*\|_F\\
& = c\sigma^W \|S_{i,1}^W - (S_{1}+e_2)\|_F + \|S_{2}-S^*\|_F\\
&\leq c\sigma^W \|S_{i,1}^W -S_{1} \|_F + c\sigma^W\|e_2\|_F + \|S_{2} - S^*\|_F\\
& \leq (c\sigma^W)^2 \|e_1\|_F + c\sigma^W\|e_2\|_F + \|S_2-S^*\|_F.
\end{align*}
At time step $t$, after $W$ rounds of communication, the difference between the sample covariance estimation of agent $i \in \mathcal{V}_a$ and the ground truth sample covariance $S^*$ is bounded by
\begin{equation}
    \|S_{i,t}^W - S^*\|_F \leq \|S_{t}-S^*\|_F + \sum_{j = 1}^{t} (c\sigma^W)^{t+1-j}\|e_j\|_F.
    \label{eqn: S bound}    
\end{equation}
From the law of large numbers, as $t \to \infty$, $S_t \to S^*$ almost surely and $e_t \to 0$ almost surely. Thus, as $t\to\infty$, the first term of \eqref{eqn: S bound} goes to zero almost surely. The second term is an instance of Lemma \ref{lemma: nedic}, where $\alpha = c\sigma^W < 1$ and $\gamma_t = \|e_t\|_F \to 0$ almost surely. 

\begin{figure}
    \centering
    \includegraphics[width =0.8 \linewidth]{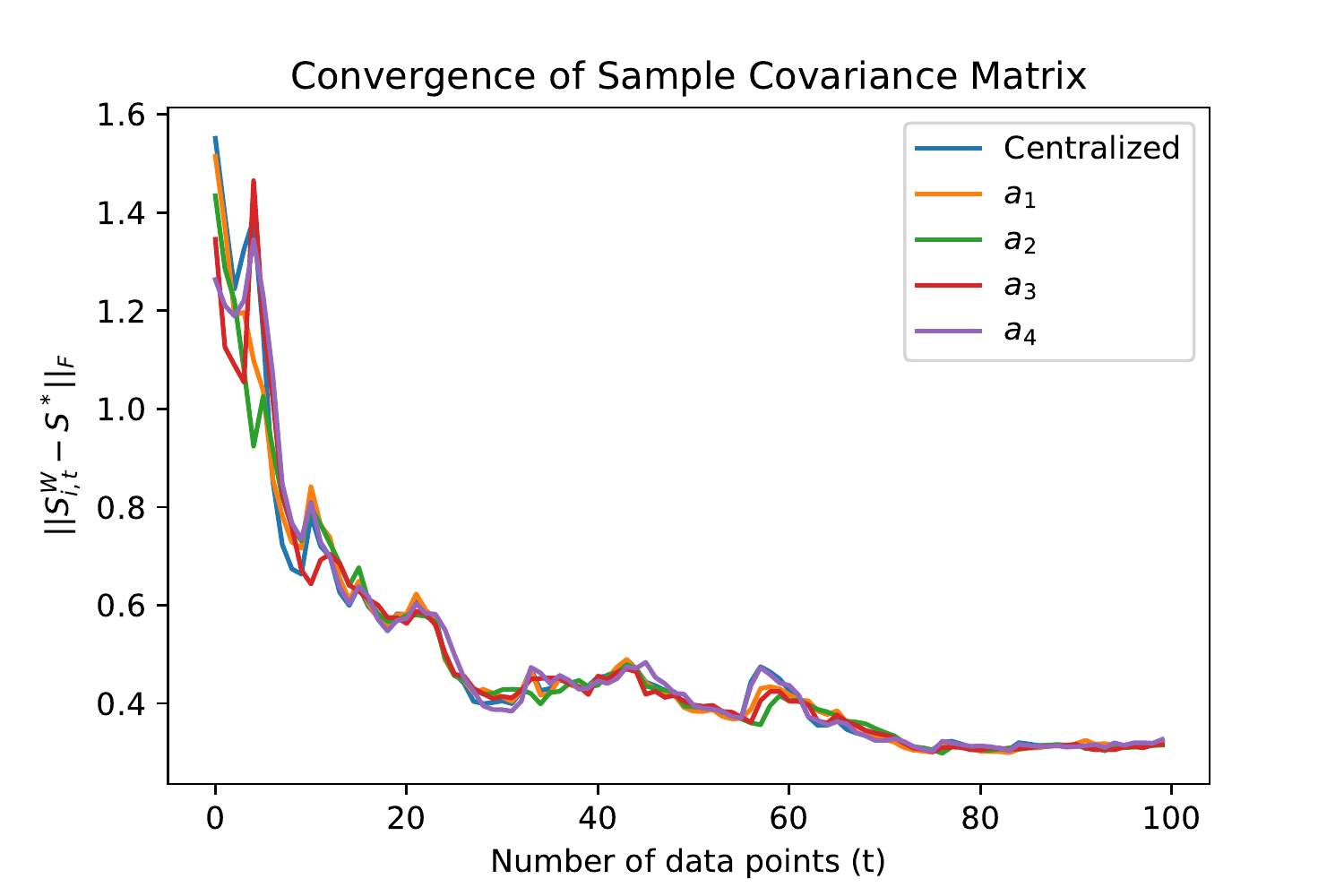}
    \caption{Demonstration of sample covariance convergence.}
    \label{fig: S cov}
\end{figure}

In Fig. \ref{fig: S cov}, we use the experiment setup as in Sec.\ref{sec: experiments}, let $W = 1$ , and demonstrate the convergence of sample covariance matrix of agents, evaluated by $\|S_{i,t}^W - S^*\|_F$.

\subsection{Proof of Corollary 1}
\label{subsec: corollary 1}
From \eqref{eqn: algo 3 bound} in Theorem \ref{thm: conv of 1}, if $\forall t\geq t_0$ and $\forall k \in \{1,2,\ldots,K\}$, $\zeta_{i,t}^k = \zeta < a^2 $, then $0<\beta <1$ . The first term $\beta^{K(t-t_0)} \|\Gamma_{i,t_0}^K - \Gamma^*\|_F$ converges to zero exponentially. 

The sum $\sum_{l=t_0+1}^{t}\beta^{K(t-l)}\|S_{i,l}^W-S^*\|_F$  of the second term is another instance of Lemma \ref{lemma: nedic}, where $\alpha = \beta^K < 1$ and $\gamma_t = \|S_{i,t}^W - S^*\|_F \to 0$ almost surely as $t \to \infty$ by Theorem \ref{thm: S conv}.

\subsection{Proof of Corollary \ref{coro: rateofconvergence}}
\label{subsec: corollary 2}
We start with the following result on the convergence of the sample covariance matrix.
\begin{lemma}[\cite{OGAMA}]\label{lemma: S concentrate}
Consider a sequence of independent and identically distributed $p$-dimensional random vectors $\{X_1, X_{2}, \ldots\}$, where each $X_t$, $t\in \mathbb{N}$, is drawn from $N(0, S^*)$. For all $t\in\mathbb{N}$, let $S_t = \frac{1}{t}\sum_{j = 1}^{t}X_jX_j^T$ be the sample covariance matrix for the data received up to time $t$. Then, for all $\Delta \in (0,\frac{1}{2})$, there exists a set of sample paths of measure 1, such that for each sample path in that set, there exists a finite time $\bar{t}$, such that for all $t\geq \bar{t}$, the sample covariance matrix satisfies the concentration inequality: 
\begin{equation*}
    \|S_t - S^*\|_F \leq 40p\max_j(S^*(j,j))\left(\frac{1}{t}\right)^{\frac{1}{2}-\Delta}.
\end{equation*}
\end{lemma}

The above result, together with \eqref{eqn: S bound}, leads directly to the following characterization of the convergence rate of the sample covariance matrices $S_{i,t}^w$ of each agent $i \in \mathcal{V}_a$.

\begin{lemma}
\label{lemma: sample cov}
Consider a sequence of independent and identically distributed $p$-dimensional random vectors $\{X_1, X_{2}, \ldots\}$, where each $X_t$, $t\in \mathbb{N}$, is drawn from $N(0, S^*)$. For all $t\in\mathbb{N}$, let $S_t = \frac{1}{t}\sum_{j = 1}^{t}X_jX_j^T$ be the sample covariance matrix for the data received up to time $t$. Then, for all $\Delta \in (0,\frac{1}{2})$, there exists a set of sample paths of measure 1, such that for each sample path in that set, there exists a finite time $\bar{t}$, such that for all $t\geq \bar{t} + 1$, the sample covariance matrix satisfies the concentration inequality: 
\begin{multline}
\label{eqn: sample cov}
\|S_{i,t}^W - S^*\|_F \leq (c\sigma^W)^{t-\bar{t}}\|S_{i,\bar{t}}^W - S_{\bar{t}}^*\|_F +  40p\max_j(S^*(j,j))\\
\cdot \Bigg(\left(\frac{1}{t}\right)^{\frac{1}{2}-\Delta} + 2\sum_{j = \bar{t}+1}^{t} (c\sigma^W)^{t+1-j}\left(\frac{1}{j-1}\right)^{\frac{1}{2}-\Delta}\Bigg).    
\end{multline}
\end{lemma}
\begin{proof}
Using triangle inequality and applying Lemma \ref{lemma: S concentrate}, $\forall t \geq \bar{t}+1$, we have 
\begin{align*}
    \|e_t\|_F &= \|S_{t} - S_{t-1}\|_F \leq \|S_{t} - S^*\|_F + \|S_{t-1}-S^*\|_F \nonumber\\
    & \leq 40p\max_j(S^*(j,j))\left(\left(\frac{1}{t}\right)^{\frac{1}{2}-\Delta} + \left(\frac{1}{t-1}\right)^{\frac{1}{2}-\Delta}\right)\nonumber\\
    &\leq 80p\max_j(S^*(j,j))\left(\frac{1}{t-1}\right)^{\frac{1}{2}-\Delta}. 
\end{align*}
To analyze the asymptotic behavior of the algorithm, we rewrite the update in \eqref{eqn: S bound}. Given the error $\|S_{i,\bar{t}}^W - S_{\bar{t}}^*\|_F$ at time $\bar{t}$, we write out the error between $S_{i,\bar{t}+1}^W$ and $S^*$:
\begin{align*}
    \|S_{i,\bar{t}+1}^W &- S^*\|_F \leq c\sigma^W\|S_{i,\bar{t}+1}^0-S_{\bar{t}+1}\|_F + \|S_{\bar{t}+1}-S^*\|_F \\
    & = c\sigma^W\|S_{i,\bar{t}}^W-(S_{\bar{t}}+e_{\bar{t}+1})\|_F + \|S_{\bar{t}+1}-S^*\|_F\\
    &\leq c\sigma^W\left(\|S_{i,\bar{t}}^W - S_{\bar{t}}\|_F+\|e_{\bar{t}+1}\|_F\right) + \|S_{\bar{t}+1}-S^*\|_F.
\end{align*}
For all $t \geq \bar{t} + 1$, we obtain
\begin{multline}
    \|S_{i,t}^W-S^*\|_F \leq (c\sigma^W)^{t-\bar{t}}\|S_{i,\bar{t}}^W - S_{\bar{t}}^*\|_F \\+ \sum_{j = \bar{t}+1}^{t}(c\sigma^W)^{t+1-j}\|e_{j}\|_F +\|S_{t} - S^*\|_F.    
    \label{eqn: S bound 2}
\end{multline}
Substituting the corresponding terms in \eqref{eqn: S bound 2} with inequalities for $\|e_t\|_F$ and in Lemma \ref{lemma: S concentrate} , we obtain \eqref{eqn: sample cov}.
\end{proof}

The concentration inequality of $\|S_{i,t}^W - S^*\|_F$ in Lemma \ref{lemma: sample cov} and the error bound $\|\Gamma_{i,t}^K - \Gamma^*\|_F$ in Theorem \ref{thm: conv of 1} immediately lead to the convergence rate of the dual variable.
\end{document}